\documentclass[acmlarge,nonacm]{acmart}
\usepackage[utf8]{inputenc}
\usepackage{mathpartir}
\usepackage{xspace}
\usepackage{pifont}
\usepackage{todonotes}
\usepackage{algorithm2e}
\usepackage{thmtools}
\usepackage{thm-restate}
\usepackage{subcaption}
\usepackage{cleveref}
\usepackage{microtype}
\usepackage{listings, lstcoq}
\usepackage{array}

\usepackage{tikz-cd}

\usepackage{prettyref}
\newcommand{\rref}[2][]{\prettyref{#2}}
\newrefformat{model}{Model\,\ref{#1}}
\newrefformat{listing}{Listing\,\ref{#1}}
\newrefformat{line}{line\,\ref{#1}}
\newrefformat{sec}{Section\,\ref{#1}}
\newrefformat{appendix}{Appendix\,\ref{#1}}
\newrefformat{def}{Definition\,\ref{#1}}
\newrefformat{thm}{Theorem\,\ref{#1}}
\newrefformat{ax}{\ref{#1}}
\newrefformat{prop}{Proposition\,\ref{#1}}
\newrefformat{lemma}{Lemma\,\ref{#1}}
\newrefformat{cor}{Corollary\,\ref{#1}}
\newrefformat{ex}{Example\,\ref{#1}}
\newrefformat{tab}{Table\,\ref{#1}}
\newrefformat{fig}{Figure\,\ref{#1}}
\newrefformat{eqn}{(\ref{#1})}

\newcommand{\leftfish}{
\begin{tikzpicture}[baseline=-2pt]
\begin{scope}[>=to]
  \draw[double, >->] (0,0) -- (0.4,0);
\end{scope}[>=to]
\end{tikzpicture}}

\newtheorem{theorem}{Theorem}
\newtheorem{definition}[theorem]{Definition}

\newcommand{\HTMLBase}{https://FormalML.github.io/CPP21/documentation/html}
\newcommand{\coqHTMLBase}{\HTMLBase}
\newcommand{\coqBaseModule}{FormalML.}

\newcommand{\fleur}{\ding{95}}
\newcommand{\coqtop}{\text{\href{https://github.com/IBM/FormalML}{\fleur}}}
\newcommand{\coqdef}[2]{\text{\href{\coqHTMLBase/\coqBaseModule#1.html\##2}{\fleur}}}

\newcommand{\libname}{\textsf{CertRL}\xspace}
\newcommand{\R}{\mathbb{R}}
\newcommand{\E}{\mathbb{E}}

\DeclareMathOperator{\argmax}{argmax}

\newcommand{\states}{S}
\newcommand{\actions}{A}
\newcommand{\bind}{\mathsf{bind}}
\newcommand{\ret}{\mathsf{ret}}
\newcommand{\funtoR}[1]{#1 \rightarrow \mathbb{R}}
\newcommand{\giry}{P}
\newcommand{\policy}{\pi}
\newcommand{\transition}{T}
\newcommand{\reward}{r}
\newcommand{\reals}{\mathbb{R}}
\newcommand{\integers}{\mathbb{Z}}

\newcommand{\bellman}{\mathbf{B}}

\newcommand{\mdplen}{n}
\newcommand{\step}{k}

\title[\libname: Formalizing Convergence Proofs for Value and Policy Iteration in Coq]{
\libname: Formalizing Convergence Proofs for Value and Policy Iteration in Coq
}

\author{Koundinya Vajjha}
\email{kov5@pitt.edu}
\affiliation{%
University of Pittsburgh
\country{USA}
}

\author{Avraham Shinnar}
\email{shinnar@us.ibm.com}
\affiliation{%
IBM Research
\country{USA}
}

\author{Barry Trager}
\email{bmt@us.ibm.com}
\affiliation{%
IBM Research
\country{USA}
}

\author{Vasily Pestun}
\email{pestun@ihes.fr}
\affiliation{%
IBM Research
\country{USA}
\&
IHES
}
\author{Nathan Fulton}
\email{nathan@ibm.com}
\affiliation{%
IBM Research
\country{USA}
}
\renewcommand{\shortauthors}{Vajjha et al.}

\setcopyright{acmlicensed}
\acmPrice{15.00}
\acmDOI{10.1145/3437992.3439927}
\acmYear{2021}
\copyrightyear{2021}
\acmSubmissionID{poplws21cppmain-p30-p}
\acmISBN{978-1-4503-8299-1/21/01}
\acmConference[CPP '21]{Proceedings of the 10th ACM SIGPLAN International Conference on Certified Programs and Proofs}{January 18--19, 2021}{Virtual, Denmark}
\acmBooktitle{Proceedings of the 10th ACM SIGPLAN International Conference on Certified Programs and Proofs (CPP '21), January 18--19, 2021, Virtual, Denmark}

\begin{document}

\begin{abstract}
Reinforcement learning algorithms solve sequential decision-making problems in probabilistic environments by optimizing for long-term reward.  The desire to use reinforcement learning in safety-critical settings inspires a recent line of work on formally constrained reinforcement learning; however, these methods place the implementation of the learning algorithm in their Trusted Computing Base. The crucial correctness property of these implementations is a guarantee that the learning algorithm converges to an optimal policy.

  This paper begins the work of closing this gap by developing a Coq formalization of two canonical reinforcement learning algorithms: value and policy iteration for finite state Markov decision processes. The central results are a formalization of the Bellman optimality principle and its proof, which uses a contraction property of Bellman optimality operator to establish that a sequence converges in the infinite horizon limit. The CertRL development exemplifies how the Giry monad and mechanized metric coinduction streamline optimality proofs for reinforcement learning algorithms. The CertRL library provides a general framework for proving properties about Markov decision processes and reinforcement learning algorithms, paving the way for further work on formalization of reinforcement learning algorithms.
\end{abstract}

\ccsdesc[500]{Software and its engineering~Formal software verification}
\ccsdesc[100]{Computer systems organization~Robotic autonomy}
\ccsdesc[500]{Computing methodologies~Value iteration}
\ccsdesc[500]{Computing methodologies~Policy iteration}

\keywords{Formal Verification, Policy Iteration, Value Iteration, Reinforcement Learning, Coinduction}

\maketitle

\section{Introduction}

Reinforcement learning (RL) algorithms solve sequential decision making problems in which the goal is to choose actions that maximize a quantitative utility function \cite{bellman1954,howard1960dynamic,Puterman1994,sutton.barto:reinforcement}.
Recent high-profile applications of reinforcement learning include beating the world's best players at Go \cite{alphagonature}, competing against top professionals in Dota \cite{openaidota}, improving protein structure prediction \cite{deepmindprotein}, and automatically controlling complex robots \cite{DBLP:journals/corr/GuHLL16}. These successes motivate the use of reinforcement learning in safety-critical and correctness-critical settings.

Reinforcement learning algorithms produce, at a minimum, a \emph{policy} that specifies which action(s) should be taken in a given state.
The primary correctness property for reinforcement learning algorithms is \emph{convergence}: in the limit, a reinforcement learning algorithm should converge to a policy that optimizes for the expected future-discounted value of the reward signal.

This paper contributes \libname, a formal proof of convergence for value iteration and policy iteration two canonical reinforcement learning algorithms \cite{bellman1954,howard1960dynamic,Puterman1994}. 
They are often taught as the first reinforcement learning methods in machine learning courses because the algorithms are relatively simple but their convergence proofs contain the main ingredients of a typical convergence argument for a reinforcement learning algorithm.

There is a cornucopia of presentations of these iterative algorithms and an equally diverse variety of proof techniques for establishing convergence. 
Many presentations state but do not prove the fact that the optimal policy of
an infinite-horizon Markov decision process with $\gamma$-discounted
reward is a \emph{stationary policy}; i.e.,
the optimal decision in a given state does not depend on the time step at which the state is encountered. 
Following this convention, this paper contributes the first formal proof that policy and value iteration converge in the limit to the optimal policy in the space of stationary policies for infinite-horizon Markov decision processes. 
In addition to establishing convergence results for the classical iterative algorithms under classical infinitary and stationarity assumptions, we also formalize an optimality result about $n$-step iterations of value iteration without a stationarity assumption. 
The former formalization matches the standard theoretical treatment, while the latter is closer to real-world implementations. We shall refer to the former case -- where the set of time steps is an infinite set -- as \textit{infinite-horizon} and the latter case as \textit{finite-horizon}.

In all cases, the convergence argument for policy/value iteration proceeds by proving that a contractive mapping converges to a fixed point and that this fixed point is an optimum.
This is typical of convergence proofs for reinforcement learning algorithms.
 \libname is intentionally designed for ongoing reinforcement learning formalization efforts.

Formalizing the convergence proof directly would require complicated and tedious $\epsilon$-hacking as well as long proofs involving large matrices.
\libname obviates these challenges using a combination of the Giry monad \cite{Giry1982ACA,JACOBS2018200} and a proof technique called
\textit{Metric coinduction} \cite{Kozencoind}.

Metric coinduction was
first identified by Kozen and Ruozzi as a way to streamline and simplify proofs of theorems about streams and stochastic processes  \cite{KozenRuozzi}.
Our convergence proofs use a specialized version of metric coinduction called contraction coinduction \cite{feys2018long} to reason about order statements concerning fixed points of contractive maps. 
Identifying a \textit{coinduction hypothesis} allows us to automatically infer that a given (closed) property holds in the limit whenever it holds \textit{ab initio}. 
The coinduction hypothesis guarantees that this property is a limiting invariant.
This is significant because the low level $\epsilon - \delta$ arguments -- typically needed to show that a given property holds of the limit -- are now neatly subsumed by a single proof rule, allowing reasoning at a higher level of abstraction. 

The \textit{finitary Giry monad} is a monad structure on the space of all finitely supported probability mass functions on a set. Function composition in the Kleisli category of this monad recovers the Chapman-Kolmogorov formula \cite{perrone2019notes,JACOBS2018200}. 
Using this fact, our formalization recasts iteration of a stochastic matrix in a Markov decision process as iterated Kleisli composites of the Giry monad, starting at an initial state. Again, this makes the presentation cleaner since we identify and reason about the basic operations of $\bind$ and $\ret$, thus bypassing the need to define matrices and matrix multiplication and substantially simplifying convergence proofs.

This paper shows how these two basic building blocks -- the finitary Giry monad and metric coinduction -- provide a compelling foundation for formalizing reinforcement learning theory.
\libname develops the basic concepts in reinforcement learning theory and demonstrates the usefulness of this library by proving several results about value and policy iteration.
\libname contains a proof of the Bellman optimality principle, an inductive relation on the optimal value and policy over the horizon length of the Markov decision process.

In practice, reinforcement learning algorithms almost always run in finite time by either fixing a run time cutoff (e.g., number training steps) or by stopping iteration after the value/policy changes become smaller than a fixed threshold.
Therefore, our development also formalizes a proof that $n$-step value iteration satisfies a finite time analogue of our convergence results.

To summarize, the \libname library contains:
\begin{enumerate}
  \item a formalization of Markov decision processes and their long-term values in terms of the finitary Giry monad,
  \item a formalization of optimal value functions and the Bellman operator,
  \item a formal proof of convergence for value iteration and a formalization of the policy improvement theorem in the case of stationary policies, and
  \item a formal proof that the optimal value function for finitary sequences satisfies the  finite time analogue of the Bellman equation. 
\end{enumerate}

Throughout the text which follows, hyperlinks to theorems, definitions
and lemmas which have formal equivalents in the Coq \cite{Coq:manual} development are
indicated by a $\coqtop$.\footnote{We recommend MacOS users view this
  document in Adobe, Firefox, or Chrome, as Preview and Safari parse
  the URLs linked to by $\coqtop$'s incorrectly.}

\libname is part of a larger project for verifying machine learning theory with applications to program synthesis. The entire development is available online at the following URL: \url{https://github.com/IBM/FormalML}.

\section{Background}

We provide a brief introduction to value/policy iteration and to the mathematical structures upon which our formalization is built: contractive metric spaces, metric coinduction, the Giry monad and Kleisli composition.

\subsection{Reinforcement Learning}

This section gently introduces the basics of reinforcement learning
with complete information about the stochastic reward and transition
functions.  In this simplified situation the 
focus of the algorithm is on optimal exploitation of reward.
This framework is also known as the stochastic optimal control problem.

We give an informal definition of Markov decision processes, trajectories, long-term values, and dynamic programming algorithms for solving Markov decision processes. Many of these concepts will be stated later in a more formal type-theoretic style; here, we focus on providing an intuitive introduction to the field.

The basic mathematical object in reinforcement learning theory is the Markov decision process.
A Markov decision process is a 4-tuple $(S, A, R, T)$ where $S$ is a set of states, $A$ is a set of actions, $R : S \times A \times S \rightarrow \mathbb{R}$ is a \emph{reward function}, and $T$ is a \emph{transition relation} on states and actions mapping each $(s,a,s') \in S \times A \times S$ to the probability that taking action $a$ in state $s$ results in a transition to $s'$.
Markov decision processes are so-called because they characterize a sequential decision-making process (each action is a decision) in which the transition structure on states and actions depends only 
on the current state.

\begin{example}[CeRtL the turtle \coqdef{converge.mdp_turtle}{CeRtL_mdp}]
Consider a simple grid world environment in which a turtle can move in cardinal directions throughout a 2D grid. The turtle receives {\tt +1} point for collecting stars, {\tt -10} for visiting red squares, and {\tt +2} for arriving at the green square. The turtle chooses which direction to move, but with probability $\frac{1}{4}$ will move in the opposite direction. For example, if the turtle takes action {\tt left} then it will go {\tt left} with probability $\frac{3}{4}$ and {\tt right} with probability $\frac{1}{4}$. The game ends when the turtle arrives at the green square.

\begin{figure}
  \centering
  \includegraphics[width=0.4\columnwidth]{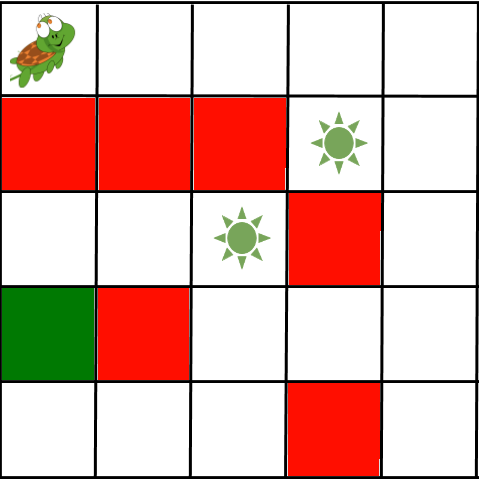} 
  \caption{An example grid-world environment \coqdef{converge.mdp_turtle}{CeRtl_grid_correct}.}
  \label{fig:picture}
\end{figure}

This environment is formulated as a Markov decision process as follows:
\begin{itemize}
  \item The set of states $S$ are the coordinates of each box \coqdef{converge.mdp_turtle}{turtle_state}: 
  $$\{(x,y) ~|~ 1 \le x \le 5 \text{ and } 1 \le y \le 5 \}\ $$
  so that $(1,1)$ is the top-left corner and $(5,5)$ is the bottom-right corner.
  \item The set of actions $A$ are $\{ {\tt up}, {\tt down}, {\tt left}, {\tt right}  \} \ \coqdef{converge.mdp_turtle}{turtle_action}$.
  \item The reward function is defined as \coqdef{converge.mdp_turtle}{turtle_reward}:
  \begin{align*}
    R(1, 4) &= 2 \\
    R(4, 2) &= 1 \\
    R(3, 3) &= 1 \\
    R(\{1,2,3\}, 2) &= -10 \\
    R(4, 3) &= -10 \\
    R(2, 4) &= -10 \\
    R(4, 5) &= -10 \\
    R(\cdot, \cdot) &= 0 \text{ otherwise} \\
  \end{align*}
  \item The transition probabilities are as described \coqdef{converge.mdp_turtle}{turtle_prob_t}; e.g., 
  \begin{align*}
  T((3,4), {\tt up}, (3,3)) &= \frac{3}{4} \\   
  T((3,4), {\tt up}, (3,5)) &= \frac{1}{4} \\
  T((3,4), {\tt up}, (\cdot, \cdot)) &= 0 ~~ \text{otherwise}
  \end{align*}
  and so on.
\end{itemize}

We implement this example in Coq as a proof-of-concept for \libname.
We first define a matrix whose indices are states $(x,y)$ and whose entries are colors $\{$red, green, star, empty$\}$. 
We then define a reward function that maps from matrix entries to a reward depending on the color of the turtle's current state. 
We also define a transition function that comports with the description given above.
At last, we prove that this combination of states, actions, transitions and rewards inhabits our {\texttt {MDP}} (standing for \textit{Markov decision process}) type.
Therefore, all of the theorems developed in this paper apply directly to our Coq implementation of the \libname Turtle environment.
\end{example}

The goal of reinforcement learning is to find a \emph{policy} ($\pi : S \rightarrow A$) specifying which action the algorithm should take in each state.
This policy should maximize the amount of reward obtained by the agent.
A policy is \emph{stationary} if it is not a function of time; i.e., if the optimal action in some state $s \in S$ is always the same and, in particular, independent of the specific time step at which $s$ is encountered.

Reinforcement learning agents optimize for a \emph{discounted sum} of rewards -- placing more emphasis on reward obtained today and less emphasis on reward obtained tomorrow. 
A constant \emph{discount factor} from the open unit interval, typically denoted by $\gamma$, quantitatively discounts future rewards and serves as a crucial hyperparameter to reinforcement learning algorithms.

Value iteration, 
invented by Bellman \cite{bellman1954}, 
is a dynamic programming algorithm that finds optimal policies to reinforcement learning algorithms by iterating a contractive mapping.
Value iteration is defined in terms of a \emph{value function} $V_\pi : S \rightarrow \mathbb{R}$, where $V_\pi(s)$ is the expected value of state $s$ when following policy $\pi$ from $s$.

\begin{algorithm}[h]
\SetAlgoLined
\KwData{\ \\ 
~~~~~Markov decision process $(S,A,T,R)$ \\ 
~~~~~Initial value function $V_0 = 0$  \\ 
~~~~~Threshold $\theta > 0$ \\
~~~~~Discount factor $0 < \gamma < 1$
}  
\KwResult{$V^*$, the value function for an optimal policy.}
\For{$\mdplen$ from $0$ to $\infty$}{
  \For{each $s \in S$}{
    $V_{\mdplen+1}[s] = \max_a \sum_{s'} T(s,a,s')(R(s,a,s') + \gamma V_{\mdplen}[s'])$
  }
  \If{$\forall s |V_{\mdplen+1}[s] - V_\mdplen| < \theta$}{
    \Return{$V_{\mdplen+1}$}
  }
}
\caption{Pseudocode for Value Iteration.}
\end{algorithm} 

The optimal policy $\pi^*$ is then obtained by 
$$\pi^*(a) = \argmax_{a \in A} \sum_{s'} T(s,a,s')(R(s,a,s') + \gamma V_{\mdplen+1}[s']).$$

Policy iteration follows a similar iteration scheme, but with a policy estimation function $Q_\pi : S \times A \rightarrow \mathbb{R}$ where $Q_\pi(s,a)$ estimates the value of taking action $a$ in state $s$ and then following the policy $\pi$.
In  \rref{sec:finite} we will demonstrate a formalized proof
that $V_{\mdplen}$ is the \emph{optimal value} function
of a length $\mdplen$ MDP; this algorithm implements the \emph{dynamic programming} principle.

\subsection{Metric and Contraction Coinduction}
Our formalization uses metric coinduction to establish convergence properties for infinite sequences.
This section recalls the Banach fixed point theorem and explains how this theorem gives rise to a useful proof technique.

A metric space $(X, d)$ is a set $X$ equipped with a function $d : X \times X \rightarrow \mathbb{R}$ satisfying certain axioms that ensure $d$ behaves like a measurement of the \emph{distance} between points in $X$.
A metric space is \emph{complete} if the limit of every Cauchy sequence of elements in $X$ is also in $X$.

Let $(X,d)$ denote a complete metric space with metric $d$. 
Subsets of $X$ are modeled by terms of the function type $\phi : X \rightarrow \mathsf{Prop}$.
Another interpretation is that $\phi$ denotes all those terms of $X$ which satisfy a particular property. 
These \emph{subsets} are also called \emph{Ensembles} in the Coq standard library.

A \emph{Lipschitz map} \coqdef{converge.LM.fixed_point}{is_Lipschitz} is a mapping that is Lipschitz continuous; 
i.e., a mapping $F$ from $(X,d_X)$ into $(Y, d_Y)$ such that for all $x_1, x_2 \in X$ there is some $K \ge 0$ such that
\[
  d_Y(F(x_1), F(x_2)) \le K d_X(x_1, x_2).
\]
The constant $K$ is called a Lipschitz constant.

A map $F : X \rightarrow X$ is called a \textit{contractive map} \coqdef{converge.LM.fixed_point}{is_contraction}, or simply a \textit{contraction}, if there exists a constant $0 \leq \gamma < 1$ such that 
\[ 
d(F(u), F(v)) \leq \gamma d(u,v) \quad \forall u,v \in X.
\]
Contractive maps are Lipschitz maps with Lipschitz constant $\gamma < 1$. 

The Banach fixed point theorem is a standard result of classical analysis which states that contractive maps on complete metric spaces have a unique fixed point.

\begin{theorem}[Banach fixed point theorem]
\label{thm:bfpt}
    If $(X,d)$ is a nonempty complete metric space and $F : X \rightarrow X$ is a contraction, then $F$ has a unique fixed point; 
    \textit{i.e.,} there exists a point $x^* \in X$ such that $F(x^*) = x^*$. This fixed point is $x^* = \lim_{\mdplen \rightarrow \infty} F^{(\mdplen)}(x_0)$ where $F^{(\mdplen)}$ stands for the $\mdplen$-th iterate of the function $F$ and $x_0$ is an arbitrary point in $X$.  
\end{theorem}

The Banach fixed point theorem generalizes to subsets of $X$.

\begin{theorem}[Banach fixed point theorem on subsets \coqdef{converge.LM.fixed_point}{FixedPoint}] 
\label{thm:bfpsubsets}
    Let $(X,d)$ be a complete metric space and $\phi$  a closed nonempty subset of $X$. Let $F : X \rightarrow X$ be a contraction and assume that $F$ preserves $\phi$. In other words, \[ 
     \phi (u) \rightarrow \phi (F(u))    
    \]
    Then $F$ has a unique fixed point in $\phi$; \textit{i.e.,} a point $x^* \in X$ such that $\phi (x^*)$ and $F(x^*) = x^*$. 
    The fixed point of $F$ is given by $x^* = \lim_{\mdplen \rightarrow \infty} F^{(\mdplen)}(x_0)$ where $F^{(\mdplen)}$ stands for the $\mdplen$-th iterate of the function $F$. 
\end{theorem}

Both the Banach fixed point theorem and the more general theorem on subsets were previously formalized in Coq by Boldo et al. \cite{BoldoElfic}. This formalization includes  definitions of Lipschitz maps and contractions. 
We make use of the fact that Boldo et al. prove the above theorem where $X$ is either a {\tt CompleteSpace} or a {\tt CompleteNormedModule}.

The fixed point of $F$ in \rref{thm:bfpsubsets} is unique, but it depends on an initial point $x_0 \in X$, which $F$ then iterates on. Uniqueness of the fixed point implies that different choices of the initial point still give the same fixed point \coqdef{converge.mdp}{fixpt_init_unique}.

To emphasize how this theorem is used in our formalization, we restate it as an inductive proof rule:
\begin{equation}\label{eq:metricc}
    \inferrule{\phi \ \mathrm{ closed} \quad \exists x_0, \phi(x_0) \quad  \phi (u) \rightarrow \phi (F(u))}{\phi (\mathsf{fix} \ F \ x_0)} \ \coqdef{converge.mdp}{metric_coinduction}
\end{equation}

This proof rule states that in order to prove
some closed $\phi$ is a property of a fixed point of $F$,
it suffices to establish the standard inductive assumptions:
that $\phi$ holds for some initial $x_0$, 
and that if $\phi$ holds at $u$ then it also holds after a single application of $F$ to $u$.
In this form, the Banach fixed point theorem is called \emph{Metric coinduction}.
The rule \rref{eq:metricc} is \emph{coinductive} because it is equivalent to the assertion that a certain coalgebra is final in a category of coalgebras. (Details are given in Section 2.3 of Kozen and Ruozzi \cite{KozenRuozzi}).

The following snippet shows how we use the Banach Fixed Point theorem as proven in \cite{BoldoElfic} as a proof rule.

\begin{coq}
Theorem metric_coinduction {phi : X -> Prop} 
 (nephi : phi init) (Hcphi : closed phi) 
 (HFphi : forall x : X, phi x -> phi (F x)):
  phi (fixpt F init). 
Proof. 
  assert (my_complete phi) 
  by (now apply closed_my_complete).
  destruct (FixedPoint K F phi fphi (ex_intro _ _ init_phi)
  H hF) as [? [Hin [? [? Hsub]]]].
  specialize (Hsub init init_phi).
  rewrite <-Hsub in Hin.
  apply Hin.
Qed.
\end{coq}

\begin{definition}[Ordered Metric Space]\label{def:ordered-met-space}
    A metric space $X$ is called an \emph{ordered metric space} if the underlying set $X$ is partially ordered and the sets $\{z \in X | z \le y \}$ and $\{z \in X | y \le z\}$ are closed sets in the metric topology for every $y \in X$.
\end{definition}

For ordered metric spaces, metric coinduction specializes to \cite[Theorem 1]{feys2018long}, which we restate as \rref{thm:contrc} below.

\begin{theorem}[Contraction coinduction]\label{thm:contrc}
    Let $X$ be a non-empty, complete ordered metric space. If $F : X \rightarrow X$ is a contraction and is order-preserving, then:
    \begin{itemize}
        \item $\forall x, F(x) \le x \Rightarrow x^* \le x$ \coqdef{converge.mdp}{contraction_coinduction_Rfct_le} and
        \item $\forall x, x \le F(x) \Rightarrow x \le x^*$ \coqdef{converge.mdp}{contraction_coinduction_Rfct_ge}
    \end{itemize}
    where $x^*$ is the fixed point of $F$.
\end{theorem}

We will use the above result to reason about Markov decision processes. However, doing so requires first setting up an ordered metric space on the function space $A \rightarrow \mathbb{R}$ where $A$ is a finite set \coqdef{converge.mdp}{Rfct}.

\subsection{The function space $A \to \mathbb{R}$.}
Let $A$ be a finite set \coqdef{utils.Finite}{Finite}.
We endow the function space $\funtoR{A}$ with a natural vector space structure
and with $L^{\infty}$ norm \coqdef{converge.mdp}{Rmax_norm}:
\begin{equation}\label{eq:rfct_norm}
    \|f\|_{\infty} = \max_{a \in A}|f(a)|    
\end{equation}

Our development establishes several important properties about this function space.
The norm \rref{eq:rfct_norm} is well-defined because $A$ is finite and furthermore induces a metric that makes $\funtoR{A}$ a metric space \coqdef{converge.mdp}{Rfct_UniformSpace}.
With this metric, the space of functions $\funtoR{A}$ is also complete \coqdef{converge.mdp}{Rfct_CompleteSpace}.
From $\R$ this metric inherits a pointwise order; \textit{viz.}, for functions $f, g : \funtoR{A}$,
\begin{align*}
    f \le g \ &\iff \ \forall a \in A, f(a) \le g(a) \ \coqdef{converge.mdp}{Rfct_le}\\    
    f \ge g \ &\iff \ \forall a \in A, f(a) \ge g(a) \ \coqdef{converge.mdp}{Rfct_ge}
\end{align*}
We also prove that the sets 
 $$\{f | f \le g\} \ \coqdef{converge.mdp}{le_closed}$$
\noindent and 
$$\{f | f \ge g\} \ \coqdef{converge.mdp}{ge_closed}$$ 
\noindent are closed in the norm topology. Our formalization of the proof of closedness for these sets relies on classical reasoning. Additionally, we rely on functional extensionality to reason about equality between functions.

We now have an ordered metric space structure on the function space $\funtoR{A}$ when $A$ is finite.
Constructing a contraction on this space will allow an application of \rref{thm:contrc}.
Once we set up a theory of Markov decision processes we will have natural examples of such a function space and contractions on it. Before doing so, we first introduce the Giry monad. 

\subsection{(Finitary) Giry Monad}\label{sec:giry}
A monad structure on the category of all measurable spaces was first described by Lawvere in \cite{lawvere1962category} and was explicitly defined by Giry in \cite{Giry1982ACA}. 
This monad has since been called the Giry monad. While the construction is very general (applying to arbitrary measures on a space), for our purposes it suffices to consider finitely supported probability measures. 

The Giry monad for finitely supported probability measures is called the \emph{finitary Giry monad}, although sometimes also goes by the more descriptive names \emph{distribution monad} and \emph{convex combination monad}.

On a set $A$, let $P(A)$ denote the set of all finitely-supported probability measures on $A$  \coqdef{converge.pmf_monad}{Pmf}.
An element of $P(A)$ is a list of elements of $A$ together with probabilities. The probability assigned to an element $a : A$ is denoted by $p(a)$. 

In our development we state this as the record 

\begin{coq}
  Record Pmf (A : Type) := mkPmf {
  outcomes :> list (nonnegreal * A);
  sum1 : list_fst_sum outcomes = R1
 }.
\end{coq}

where \texttt{outcomes} stores all the entries of the type \texttt{A} along with their atomic probabilities. The field \texttt{sum1} ensures that the probabilities sum to 1. 

The Giry monad is defined in terms of two basic operations associated to this space: 
\begin{align*}
\mathsf{ret}: \ &A \rightarrow P(A) \ \coqdef{converge.pmf_monad}{Pmf_pure} \\   
& a \mapsto \lambda x:A, \ \delta_a(x)
\end{align*}
where $\delta_a(x) = 1$ if $a=x$ and $0$ otherwise. The other basic operation is
\begin{align*}
    \mathsf{bind} : \ &P(A) \rightarrow (A \rightarrow P(B)) \rightarrow P(B) \ \coqdef{converge.pmf_monad}{Pmf_bind} \\
    \mathsf{bind} \ &p \ f = \lambda \ b:B, \ \sum_{a \in A} f(a)(b) * p(a)
\end{align*}
In both cases the resulting output is a probability measure.
The above definition is well-defined because we only consider finitely-supported probability measures.
A more general case is obtained by replacing sums with integrals.

The definitions of $\bind$ and $\ret$ satisfy the following properties: 
\begin{align}
  \bind\ (\ret\ x)\ f & = \ret(f(x)) \ \coqdef{converge.pmf_monad}{Pmf_bind_of_ret} \\
  \bind\ p\ (\lambda x,\, \delta_x) & = p \ \coqdef{converge.pmf_monad}{Pmf_ret_of_bind}\\
  \bind\ (\bind\ p\ f)\ g & = \bind\ p\ (\lambda x, \bind\ (f x)\ g) \ \coqdef{converge.pmf_monad}{Pmf_bind_of_bind}
\end{align}

These \emph{monad laws} establish that the triple $(P,\bind,\ret)$ forms a monad. 

The Giry monad has been extensively studied and used by various authors because it has several attractive qualities that simplify (especially formal) proofs.
First, the Giry monad naturally admits a denotational monadic semantics for certain probabilistic programs \cite{DBLP:conf/popl/RamseyP02,DBLP:conf/lics/JonesP89,DBLP:conf/haskell/ScibiorGG15,audebaud2009proofs}. 
Second, it is useful for rigorously formalizing certain informal arguments in probability theory by providing a means to perform \textit{ad hoc} notation overloading \cite{DBLP:journals/corr/abs-1911-00385}. 
Third, it can simplify certain constructions such as that of the product measure \cite{eberl2015pdf}. 

\libname uses the Giry monad as a substitute for the stochastic matrix associated to a Markov decision process. 
This is possible because the Kleisli composition of the Giry monad recovers the Chapman-Kolmogorov formula \cite{perrone2018categorical,perrone2019notes}. 
The Kleisli composition is the \emph{fish} operator in Haskell parlance.

\subsubsection{Kleisli Composition}
Reasoning about probabilistic \emph{processes} requires composing probabilities.
The Chapman-Kolmogorov formula is a classical result in the theory of Markovian processes that states the probability of transition from one state to another through two steps can be obtained by summing up the probability of visiting each intermediate state.
This application of the Chapman-Kolmogorov formula plays a fundamental role in the study of Markovian processes, but requires formalizing and reasoning about matrix operations.

Kleisli composition provides an alternative and more elegant mechanism for reasoning about compositions of probabilistic choices. This section defines and provides an intuition for Kleisli composition.

Think of $P(A)$ as the random elements of $A$ (\cite[page 15]{perrone2018categorical}).
In this paradigm, the set of maps $A \rightarrow P(B)$ are simply the set of maps with a random outcome. When $P$ is a monad, such maps are called Kleisli arrows of $P$.  

In terms of reinforcement learning, a map $f : A \rightarrow P(B)$ is a rule which takes a state $a:A$ and gives the probability of transitioning to state $b:B$. 
Suppose now that we have another such rule $g : B \rightarrow P(C)$. 
Kleisli composition puts $f$ and $g$ together to give a map $(f \leftfish g) : A \rightarrow P(C)$.
It is defined as:
\begin{align}
    f \leftfish g &:= \lambda x:A, \bind \ (f \ x)\ g \label{eq:firstguy} \\
             &= \lambda x:A, (\lambda c:C, \sum_{b : B} g(b)(c) *f(x)(b)) \\
             &= \lambda (x:A) \ (c:C), \sum_{b : B} f(x)(b)*g(b)(c)\label{eq:chap-kol}
\end{align}

The motivation for \rref{eq:firstguy}--\rref{eq:chap-kol} is intuitive.
In order to start at $x:A$ and end up at $c:C$ by following the rules $f$ and $g$, one must first pass through an intermediate state $b:B$ in the codomain of $f$ and the domain of $g$. 
The probability of that point being any particular $b:B$ is $$f(x)(b)*g(b)(c).$$ 
So, to obtain the total probability of transitioning from $x$ to $c$, simply sum over all intermediate states $b:B$. This is exactly \rref{eq:chap-kol}.
We thus recover the classical Chapman-Kolmogorov formula, but as a Kleisli composition of the Giry monad. 
This obviates the need for reasoning about operators on linear vector spaces, thereby substantially simplifying the formalization effort.

Indeed, if we did not use Kleisli composition, we would have to associate a stochastic transition matrix to our Markov process and manually prove various properties about stochastic matrices which can quickly get tedious. With Kleisli composition however, our proofs become more natural and we reason closer to the metal instead of adapting to a particular representation. 

\section{The \libname Library}
\label{sec:library}

\libname contains a formalization of Markov decision processes, a definition of the Kleisli composition specialized to Markov decision processes, a definition of the long-term value of a Markov decision process, a definition of the Bellman operator, and a formalization of the operator's main properties. 

Building on top of its library of results about Markov decision processes, \libname contains proofs of our main results:
\begin{enumerate}
  \item the (infinite) sequence of value functions obtained by value iteration converges in the limit to a global optimum assuming stationary policies,
  \item the (infinite) sequence of policies obtained by policy iteration converges in the limit to a global optimum assuming stationary policies, and
  \item the optimal value function for  Markov decision process of length $\mdplen$ is computed inductively
    by application of Bellman operator, \rref{sec:finite}.
\end{enumerate}

The following sections describe the above results more carefully. 

\subsection{Markov Decision Processes} 
We refer to \cite{Puterman1994} for detailed presentation of the
theory of Markov decision processes.
Our formalization considers the theory of \emph{infinite-horizon discounted Markov decision processes with deterministic stationary policies}.

We now elaborate on the above definitions and set up relevant notation. Our presentation will be type-theoretic in nature, to reflect the formal development. The exposition (and \libname formalization) closely
follows the work of Frank Feys, Helle Hvid Hansen, and Lawrence Moss \cite{feys2018long}.

\subsubsection{Basic Definitions}
\begin{definition}[Markov Decision Process \coqdef{converge.mdp}{MDP}]\label{def:MDP}
  A Markov decision process consists of the following data:
  \begin{itemize}
    \item A nonempty finite type $\states$ called \emph{the set of states}.\footnote{There are various definitions of finite.  Our mechanization uses surjective finiteness (the existence of a surjection from a bounded set of natural numbers)\coqdef{utils.Finite}{Finite}, and assumes that there is a decidable equality on $\states$.  This pair of assumptions is equivalent to bijectve finitness.}
    \item For each state $s : \states$, a nonempty finite type $\actions(s)$ called \emph{the type of actions} available at state $s$. This is modelled as a dependent type.
    \item A \emph{stochastic transition structure} $\transition: \prod_{s : S} (\actions(s) \to \giry(\states))$. Here $\giry(\states)$ stands for the set of all probability measures on $\states$, as described in Section \ref{sec:giry}.
    \item A \emph{reward function} $\reward: \prod_{s : S} (\actions(s) \to \states \to \reals)$ 
      where $\reward(s, a, s')$ is the reward
    obtained on transition from state $s$ to state $s'$ under action $a$. 
  \end{itemize}
\end{definition}
From these definitions it follows that the rewards are bounded in absolute value: since the state and action spaces are finite, there exists a constant $D$ such that 
\begin{equation}\label{eq:bdd_rewards}
  \forall (s \ s' : \states), (a : A(s)), |\reward (s, a, s')| \le D \ \coqdef{converge.mdp}{bdd}
\end{equation}

\begin{definition}[Decision Rule / Policy]\label{def:dec_rule}
  Given a Markov decision process with state space $\states$ and action space $\prod_{s : \states}\actions(s)$,
  \begin{itemize}
    \item   A function $ \policy: \prod_{s:\states} \actions(s)$ is called a \emph{decision rule}  \coqdef{converge.mdp}{dec_rule}. The decision rule is \emph{deterministic} \footnote{if the decision rule takes a state and returns a probability distribution on actions instead, it is called \emph{stochastic}.}. 
    \item   A \emph{stationary policy} is an infinite \emph{sequence} of decision rules: $(\pi,\pi,\pi,...)$ \coqdef{converge.mdp}{policy}. \emph{Stationary} implies that the same decision rule is applies at each step.
  \end{itemize}
\end{definition}

This policy $\policy$ induces a stochastic dynamic process on $\states$
evolving in discrete time steps $\step \in \integers_{\geq 0}$.
In this section we consider only stationary policies, and therefore
use the terms \emph{policy} and \emph{decision rule} interchangeably.

\subsubsection{Kleisli Composites in a Markov Decision Process}

Note that for a fixed decision rule $\policy$, we get a Kleisli arrow $\transition_{\policy}: \states \to \giry(\states)$ defined as $\transition_{\policy}(s) = \transition(s)(\pi(s))$.

Conventionally, $\transition_{\policy}$ is represented as a row-stochastic matrix
$(T_{\policy})^{s}{}_{s'}$ that acts on the probability co-vectors from the right, so that the row $s$
of $T_{\policy}$ corresponding to state $s$
encodes the probability distribution of states $s'$ after a transition from the state $s$. 

Let $p_{\step} \in P(\states)$ for $\step \in \integers_{\geq 0}$ denote a probability distribution on $S$ evolving
under the policy stochastic map $\transition_{\policy}$ after $\step$ transition steps, so
that $p_0$ is the initial probability distribution on $S$ (the initial distribution is usually taken to be $\ret \ s_0$ for a state $s_0$). These are related by 
\begin{equation}\label{eq:transition}
  p_{\step} = p_0 \transition_{\policy}^{\step}  
\end{equation}
In general (if $p_0 = \ret \ s_0$) the number $p_{\step}(s)$ gives the probability that starting out at $s_0$, one ends up at $s$ after $\step$ stages. So, for example, if $\step=1$, we recover the stochastic transition structure at the end of the first step \coqdef{converge.mdp}{bind_stoch_iter_1}. 

Instead of representing $T_{\policy}^{\step}$ as an iterated product of a stochastic matrix in our formalization, we recognize that \rref{eq:transition} states that $p_{\step}$ is the $\step$-fold iterated Kleisli composite of $T_{\pi}$ applied to the initial distribution $p_0$ \coqdef{converge.mdp}{bind_stoch_iter}.
\begin{equation}
  p_{\step} = (p_0 \leftfish \underbrace{T_{\policy} \leftfish \dots \leftfish T_{\policy}}_{\step \ \text{times}} )
\end{equation}
Thus, we bypass the need to define matrices and matrix multiplication entirely in the formalization. 

\subsubsection{Long-Term Value of a Markov Decision Process}
 Since the transition from one state to another by an action is governed by a probability distribution $T$, there is a notion of expected reward with respect to that distribution.

\begin{definition}[Expected immediate reward]
  For a Markov decision process, 
  \begin{itemize}
  \item An \emph{expected immediate reward} to be obtained in the transition under action $a$
    from state $s$ to state $s'$
     is a function $\bar \reward: \states \to \actions \to \reals$ computed by averaging the
    reward function over the stochastic transition map to a new state $s'$
    \begin{equation}\label{eq:expt_im_rwd}
      \bar \reward(s, a) := \sum_{s' \in \states} \reward(s, a, s') \transition(s,a) (s')
    \end{equation}
    \item An \emph{expected immediate reward under a decision rule} $\policy$, denoted $\bar \reward_\policy: S \to \reals$ is defined to be:  
    \begin{equation} \label{eq:step_expt_reward}
      \bar \reward_{\policy} (s) := \bar \reward(s,\policy(s)) \  \coqdef{converge.mdp}{step_expt_reward} 
    \end{equation}
   That is, we replace the action argument in \rref{eq:expt_im_rwd} by the action prescribed by the decision rule $\policy$.
    \item The \emph{expected reward} at time step $\step$ of a Markov decision process starting at initial state $s$, following policy $\pi$ is defined as the expected value of the reward with respect to the $\step$-th Kleisli iterate of $T_{\policy}$ starting at state $s$.
    \[  r^{\policy}_{\step}(s) := \mathbb{E}_{T_\policy^{\step}(s)}\left[\bar \reward_{\policy} \right] = \sum_{s' \in S}\left[ \bar \reward_{\policy}(s') T_{\policy}^{\step}(s)(s')\right] \ \coqdef{converge.mdp}{expt_reward}\] 
  \end{itemize}
  \end{definition}

The long-term value of a Markov decision process under a policy $\policy$ is defined as follows:

\begin{definition}[Long-Term Value]\label{def:ltv}
  Let $\gamma \in \reals, 0 \le \gamma < 1$ be a \emph{discount factor}, and $\policy = (\policy,\policy,\dots)$ be a stationary policy. Then $V_{\policy}: \states \to \reals$ is given by
  \begin{equation}
    V_{\policy}(s) = \sum_{\step=0}^{\infty} \gamma^{\step} r^{\policy}_{\step}(s) \ \coqdef{converge.mdp}{ltv}
  \label{eq:policyvalue}
  \end{equation}
\end{definition}

The rewards being bounded in absolute value implies that the long-term value function $V_{\policy}$ is well-defined for every initial state \coqdef{converge.mdp}{ex_series_ltv}.

It can be shown by manipulating the series in \rref{eq:policyvalue} that the long-term value satisfies the Bellman equation:
\begin{align}
V_{\policy}(s) &= \bar v(s,\policy(s)) + \gamma \sum_{s' \in S} V_{\policy}(s') T_{\policy}(s)(s') \ \coqdef{converge.mdp}{ltv_corec}  \label{eq:ltv_corec}\\ 
&= \bar \reward_\policy(s) + \gamma \mathbb{E}_{T_\pi(s)} \left[V_{\policy}\right]
\end{align}
\begin{definition}
Given a Markov decision process, we define the \emph{Bellman operator} as
\begin{align}
  \bellman_{\policy}:& (\states \to \R) \to (\states \to \R)   \\
  & W \mapsto \bar \reward_\policy(s) + \gamma  \mathbb{E}_{T_\pi(s)} W
\end{align}
\end{definition}

\begin{theorem}[Properties of the Bellman Operator \coqdef{converge.mdp}{bellman_op_monotone_le}]\label{thm:bellman_op_prop}
The Bellman operator satisfies the following properties:
\begin{itemize}
  \item As is evident from \rref{eq:ltv_corec}, the long-term value $V_{\policy}$ is the fixed point of the operator $\bellman_\policy$ \coqdef{converge.mdp}{ltv_bellman_op_fixpt}.
  \item  The operator $\bellman_\policy$ (called the Bellman operator) is a contraction in the norm \rref{eq:rfct_norm} \coqdef{converge.mdp}{is_contraction_bellman_op}.
  \item The operator $\bellman_\policy$ is a monotone operator. That is, 
  \[\forall s, W_1(s) \le W_2(s) \Rightarrow \forall s, \bellman_\policy(W_1)(s) \le \bellman_\policy(W_2)(s) \]
\end{itemize}
\end{theorem}

The Banach fixed point theorem now says that $V_{\policy}$ is the unique fixed point of this operator.

Let $V_{\policy, \mdplen}: \states \to \reals$ be the $\mdplen$-th iterate of the Bellman operator $B_\policy$. It can be computed by the recursion
relation 
\begin{equation}
  \begin{aligned}
   & V_{\policy, 0}(s_0) = 0 \\
   &  V_{\policy, \mdplen + 1}(s_0) = \bar \reward_{\policy}(s_0) + \gamma \mathbb{E}_{T_\pi(s_0)}V_{\policy, \mdplen} \qquad \mdplen \in \integers_{\ge 0} 
 \end{aligned}
 \label{eq:trecursion}
\end{equation}
where $s_0$ is an arbitrary initial state. The first term in the reward function $V_{\policy, \mdplen+1}$ for the process of length $\mdplen+1$ is the
sum of the reward collected in the first step (\emph{immediate reward}), and the remaining
total reward obtained in the subsequent process of length $\mdplen$ (\emph{discounted future reward}).
The $\mdplen$-th iterate is also seen to be equal to the $\mdplen$-th partial sum of the series \rref{eq:policyvalue} \coqdef{converge.mdp}{ltv_part}.

The sequence of iterates $\{V_{\policy, \mdplen}\}|_{n = 0,1,2, \dots}$ is convergent and equals $V_{\policy}$, by the Banach fixed point theorem. 
\begin{equation}
  V_{\policy} = \lim_{\mdplen \to \infty} V_{\policy, \mdplen} \ \coqdef{converge.mdp}{bellman_op_iterate}
\end{equation} 



\subsection{Convergence of Value Iteration}

In the previous subsection we defined the long-term value function
$V_{\policy}$ and showed that it is the fixed point of the Bellman operator. It is also the pointwise limit of the iterates $V_{\policy,\mdplen}$, which is the expected value of all length $\mdplen$ realizations of the Markov decision process following a fixed stationary policy $\policy$. 

We note that the value function $V_{\policy}$ induces a partial order on the space of all decision rules; with $\sigma \le \tau$ if and only if $V_{\sigma} \le V_{\tau}$ \coqdef{converge.mdp}{policy_le}.

The space of all decision rules is finite because the state and action spaces are finite \coqdef{converge.mdp}{dec_rule_finite}.

The above facts imply the existence of a decision rule (stationary policy) which maximizes the long-term reward. We call this stationary policy the \emph{optimal policy} and its long-term value the \emph{optimal value function}. 
\begin{equation}\label{eq:max_ltv}
    V_*(s) = \max_{\policy} \{V_\policy(s)\} \ \coqdef{converge.mdp}{max_ltv} 
\end{equation}

The aim of reinforcement learning, as we remarked in the introduction, is to have tractable algorithms to find the optimal policy and the optimal value function corresponding to the optimal policy. 

Bellman's \emph{value iteration} algorithm is such an algorithm, which is known to converge asymptotically to the optimal value function. In this section we describe this algorithm and formally prove this convergence property.

\begin{definition}\label{def:bellmanopt}
   Given a Markov decision process we define the \emph{Bellman optimality operator} as:
   \begin{align*}
    \hat \bellman:& (\states \to \reals) \to (\states \to \reals) \\
    &W \mapsto \lambda s, \max_{a \in A(s)}
    \left(\bar \reward(s,a)  + \gamma \mathbb{E}_{\transition(s,a)} [W]\right) \ \coqdef{converge.mdp}{bellman_max_op} 
   \end{align*}
\end{definition}

\begin{theorem}\label{thm:bellman_max_op_prop}
    The Bellman optimality operator $\hat \bellman$ satisfies the following properties:
    \begin{itemize}
      \item  The operator $\hat \bellman$  is a contraction with respect to the $L^{\infty}$ norm (\ref{eq:rfct_norm}) \coqdef{converge.mdp}{is_contraction_bellman_max_op}.
      \item The operator $\hat \bellman$ is a monotone operator. That is, 
      \[\forall s, W_1(s) \le W_2(s) \Rightarrow \forall s, \hat \bellman(W_1)(s) \le \hat \bellman(W_2)(s) \ \coqdef{converge.mdp}{bellman_max_op_monotone_le}\]
    \end{itemize}
\end{theorem}

Now we move on to proving the most important property of $\hat \bellman$:
the optimal value function $V_*$ is a fixed point of $\hat \bellman$.

By Theorem \ref{thm:bellman_max_op_prop} and the Banach fixed point theorem, we know that the fixed point of $\hat \bellman$ exists. Let us denote it $\hat V$. Then we have: 

\begin{theorem}[Lemma 1 of \cite{feys2018long} \coqdef{converge.mdp}{ltv_Rfct_le_fixpt}]\label{thm:ltv_Rfct_le_fixpt}
    For every decision rule $\sigma$, we have $V_\sigma\le \hat V$. 
\end{theorem}
\begin{proof}
    Fix a policy $\sigma$. 
    Note that for every $f : S \rightarrow \reals$, we have $\bellman_{\sigma}(f) \le \hat \bellman (f)$\coqdef{converge.mdp}{bellman_op_bellman_max_le}. In particular, applying this to $f = V_\sigma$ and using Theorem \ref{thm:bellman_op_prop}, we get that $V_\sigma = \bellman_\sigma(V_\sigma)\le \hat \bellman(V_\sigma)$. Now by contraction coinduction (Theorem \ref{thm:contrc} with $F = \hat \bellman$ along with Theorem \ref{thm:bellman_max_op_prop}) we get that $V_\sigma \le \hat V$.
\end{proof}

Theorem \ref{thm:ltv_Rfct_le_fixpt} immediately implies that $V_* \le \hat V$. 

To go the other way, we introduce the following policy, called the \emph{greedy} decision rule. 

\begin{equation}\label{eq:greedy} 
    \sigma_*(s) := \argmax_{a \in A(s)}\left( \bar \reward (a,s) + \gamma \mathbb{E}_{T(s,a)}[\hat V]\right) \ \coqdef{converge.mdp}{greedy}
\end{equation}

We now have the following theorem:

\begin{restatable}[Proposition 1 of \cite{feys2018long} \coqdef{converge.mdp}{exists_fixpt_policy}]{theorem}{exists} \label{thm:exists_fixpt_policy}
   The greedy policy is the policy whose long-term value is the fixed point of  $\hat \bellman$: 
   \[V_{\sigma_*} = \hat V\]
\end{restatable}
\begin{proof}
    We observe that $\bellman_{\sigma_*}(\hat V) = \hat V$ \coqdef{converge.mdp}{exists_fixpt_policy_aux}.    Thus, $\hat V \le \bellman_{\sigma_*}(\hat V) $.
    Note that we have $V_{\sigma_*}$ is the fixed point of $B_{\sigma_*}$ by Theorem \ref{thm:bellman_op_prop}. Now applying contraction coinduction with $F = \bellman_{\sigma_*}$, we get $\hat V \le V_{\sigma_*}$. 
    From Theorem \ref{thm:ltv_Rfct_le_fixpt} we get that $V_{\sigma_*} \le \hat V$.
\end{proof}

Theorem \ref{thm:exists_fixpt_policy} implies that $V_* \ge \hat V$ and so we conclude that $V_* = \hat V$ \coqdef{converge.mdp}{max_ltv_eq_fixpt}.

Thus, the fixed point of the optimal Bellman operator $\hat \bellman$ exists and is equal to the optimal value function. 

 Stated fully, value iteration proceeds by:
\begin{enumerate}
  \item Initialize a value function $V_0 : \states \to \reals$. 
  \item Define $V_{\mdplen+1} = \hat \bellman V_\mdplen$ for $\mdplen\ge 0$. At each stage, the following policy is computed
  \[
    \policy_\mdplen(s) \in \argmax_{a \in A(s)}\left( \bar \reward(s,a) + \gamma \mathbb{E}_{T(s,a)}[V_\mdplen]\right) 
  \]
\end{enumerate}

By the Banach Fixed Point Theorem, the sequence $\{V_\mdplen\}$ converges to the optimal value function $V_*$ \coqdef{converge.mdp}{bellman_iterate}. In practice, one repeats this iteration as many times as needed until a fixed threshold is breached.

In \rref{sec:finite} we explain and provide a formalized proof of the \emph{dynamic programming principle}: the value function $V_{\mdplen}$
is equal to the \emph{optimal value function} of a finite-horizon MDP of length $\mdplen$ with a possibly non-stationary optimal policy.


\subsection{Convergence of Policy Iteration}
The convergence of value iteration is asymptotic, which means the iteration is continued until a fixed threshold is breached. 
Policy iteration is a similar iterative algorithm that benefits from a more definite stopping condition. Define the \emph{$Q$ function} to be:
\[ 
    Q_\policy(s,a) := \bar \reward(s,a) + \gamma \mathbb{E}_{T(s,a)}[V_{\policy}].
\]
The policy iteration algorithm proceeds in the following steps:
\begin{enumerate}
    \item Initialize the policy to $\pi_0$. 
    \item Policy evaluation: For $\mdplen \ge 0$, given $\pi_\mdplen$, compute $V_{\pi_\mdplen}$. 
    \item Policy improvement: From $V_{\pi_\mdplen}$, compute the greedy policy:
    \[
      \pi_{\mdplen+1}(s) \in \argmax_{a \in A(s)}\left[ Q_{\pi_\mdplen}(s,a) \right]  
    \] 
    \item Check if $V_{\pi_\mdplen} = V_{\pi_{\mdplen+1}}$. If yes, stop.
    \item If not, repeat (2) and (3).
\end{enumerate}

This algorithm depends on the following results for correctness. We follow the presentation from \cite{feys2018long}.

\begin{definition}[Improved policy \coqdef{converge.mdp}{improved_tot}] \label{def:improved_policy}    
A policy $\tau$ is called an improvement of a policy $\sigma$ if for all $s \in S$ it holds that 
    \[ 
    \tau(s) = \argmax_{a \in A(s)}\left[ Q_{\sigma}(s,a)] \right]
    \]
\end{definition}
So, step (2) of the policy iteration algorithm simply constructs an improved policy from the previous policy at each stage. 

\begin{theorem}[Policy Improvement Theorem] \label{thm:policy_improvement}
    Let $\sigma$ and $\tau$ be two policies. 
    \begin{itemize}
        \item If $\bellman_\tau V_{\sigma} \ge \bellman_\sigma V_{\sigma}$ then $V_\tau \ge V_\sigma$ \coqdef{converge.mdp}{policy_improvement_1}.
        \item If $\bellman_\tau V_{\sigma} \le \bellman_\sigma V_{\sigma}$ then $V_\tau \le V_\sigma$ \coqdef{converge.mdp}{policy_improvement_2}.
    \end{itemize}
\end{theorem}

Using the above theorem, we have: 
\begin{theorem}[Policy Improvement Improves Values \coqdef{converge.mdp}{improved_has_better_value}] 
    If $\sigma$ and $\tau$ are two policies and if $\tau$ is an improvement of $\sigma$, then we have $V_{\tau} \ge V_{\sigma}$. 
\end{theorem}
\begin{proof}
    From \rref{thm:policy_improvement}, it is enough to show $\bellman_\tau V_{\sigma} \ge \bellman_\sigma V_{\sigma}$.
    We have that $\tau$ is an improvement of $\sigma$.
    \begin{align}
        \tau(s) &= \argmax_{a \in A(s)}\left[ Q_{\sigma}(s,a) \right] \\
        &= \argmax_{a \in A(s)}\left[ \bar \reward(s,a) + \gamma \mathbb{E}_{T(s,a)}[V_{\sigma}] \right]  \label{eq:tau_prop}
    \end{align}

    Note that 
    \begin{align*}
        \bellman_\tau V_{\sigma} &= \bar \reward(s,\tau(s)) + \gamma  \mathbb{E}_{T(s,\tau(s))} [V_{\sigma}] \\
        &= \max_{a \in A(s)}\left[\bar \reward(s,a) + \gamma \mathbb{E}_{T(s,a)}[V_{\sigma}]\right] \quad \mathrm{by} \ \mathrm{\rref{eq:tau_prop}}\\
        &\ge \bar \reward(s,\sigma(s)) + \gamma \mathbb{E}_{T(s,\sigma(s))}[V_{\sigma}] \\
        &= \bellman_\sigma V_\sigma
    \end{align*}
\end{proof}

In other words, since $\policy_{\mdplen+1}$ is an improvement of $\policy_\mdplen$ by construction, the above theorem implies that  $V_{\policy_\mdplen} \le V_{\policy_{\mdplen+1}}$. This means that $\policy_\mdplen \le \policy_{\mdplen+1}$. 

Thus, the policy constructed in each stage in the policy iteration algorithm is an improvement of the policy in the previous stage. Since the set of policies is finite \coqdef{converge.mdp}{dec_rule_finite}, this policy list must at some point stabilize. 
Thus, the algorithm is guaranteed to terminate.

In \rref{sec:finite} we will provide formalization
of the statement that $\policy_{\mdplen}$ is actually
the \emph{optimal policy} to follow
for an MDP process of any finite length
at that timestep when $\mdplen$ steps remain towards the end of the process.

\subsection{Optimal policy for finite time horizon Markov decision processes\label{sec:finite}}
All results up to this subsection were stated in terms of the convergences of infinite sequences of states and actions.
Stating convergence results in terms of the limits of infinite sequences is not uncommon in texts on reinforcement learning; however, in practice, reinforcement learning algorithms are always run for a finite number of steps.
In this section we consider decision processes of finite length and do not impose an assumption that the optimal policy is stationary.

Let $V_{\bar \policy}$ denote the value function of Markov decision process
for a finite
sequence of policies $\bar \policy = \policy_0::\policy_1::\policy_2:: \dots
\policy_{\mdplen-1}$ of length $\mdplen = \mathsf{len}(\bar \policy)$. 
Denote by $p_0$ the probability distribution over the initial state
at the start of the process. 

We define the probability measure at step $\step$ in terms of Kleisli iterates for \emph{each decision rule} 
$\policy_{i}$ for $i$ in $0 \dots (\step-1)$:
\begin{equation} 
  p_0 T_{\bar \policy[:\step]} := (p_0 \leftfish T_{\policy_0} \dots \leftfish
  T_{\policy_{\step-1}}) \ \coqdef{converge.finite_time}{kliesli_iter_left_vector}
\end{equation}

Below, we will use the pairing notation (bra-ket notation)
\begin{equation}
 \langle p | V \rangle := \mathbb{E}_{p} [V]
\end{equation}
between a probability measure $p$ on a finite set $S$ and a function $V : S \to \reals$,
so that $| V \rangle $ is an element of the vector space
of real valued functions on $S$, $\langle p |$ is a linear form on this vector space associated to a
probablity measure $p$ on $S$, and $\langle p | V \rangle$ denotes evaluation of a linear form $\langle p|$ on
a vector $|V\rangle$.

\begin{definition}[expectation value function of MDP of length $\mdplen = \mathsf{len}(\bar \policy)$ over the initial probability distribution $p_0$]
\label{def:Bellmannstep}
  \begin{equation}
\langle p_0 | V_{\bar \policy} \rangle = \sum_{\step = 0}^{\mdplen-1} \gamma^{\step}   \langle p_0 T_{\bar \policy[:\step]}  | \bar \reward_{\policy_{\step}} \rangle \ \coqdef{converge.finite_time}{ltv_fin}
 \label{eq:finite-time-value}
 \end{equation}
\end{definition}

\rref{def:Bellmannstep} implies the recursion relation
\begin{equation}
  \langle p_0 |  V_{\policy_0::\mathsf{tail}} \rangle =
  \langle p_0 | \bar \reward_{\policy_0} + \gamma \transition_{\policy_0} V_{\mathsf{tail}} \rangle \ \coqdef{converge.finite_time}{ltv_fin_cons} \qquad \mdplen \in \integers_{\ge 0} 
 \label{eq:trecursion1}
\end{equation}
where $\bar \policy = \policy_0 :: \mathsf{tail}$.

Let $\hat V_{*, \mdplen}$ be the optimal value function of the Markov decision process of length
$\mdplen$ on the space of all policy sequences of length $\mdplen$:
\begin{equation}
  \hat V_{*, \mdplen} := \sup_{\bar{\pi} | \mathsf{len}(\pi) = \mdplen} V_{\bar \policy} \ \coqdef{converge.finite_time}{max_ltv_fin}
\end{equation}

Let $\hat V_{\policy_0 :: *, \mdplen + 1}$ be the optimal value function of the Markov decision process
of length $\mdplen + 1$ on the space of all policy sequences of length
$\mdplen + 1$ whose initial term is $\policy_0$.
Using the relation \rref{eq:trecursion1} and that 
\begin{equation}
  \sup_{\policy_0 :: \mathsf{tail}} V_{\policy_0::\mathsf{tail}, \mdplen+1} = \sup_{\policy_0} \sup_{\mathsf{tail}} V_{\policy_0::\mathsf{tail}, \mdplen+1} \ \coqdef{converge.finite_time}{Rmax_list_dec_rule_split}
\label{eq:optimal-value-finite-time}
\end{equation}
we find 
\begin{equation}
  \begin{aligned}
    \langle p_0 | \hat V_{*, \mdplen + 1} \rangle = \sup_{\policy_0 \in \prod_{S} A(s) } \langle p_0 |\bar \reward_{\policy_0} + \gamma T_{\policy_0} \hat V_{*, \mdplen}
    \rangle \ \coqdef{converge.finite_time}{max_ltv_corec} \qquad \mdplen \in \integers_{>=0} \end{aligned}
\label{eq:trecursion2}
\end{equation}
with the initial term of the sequence $V_{*, 0} = 0$. 

The result \rref{eq:trecursion2} can be formulated as follows
\begin{theorem}[Bellman's finite-time optimal policy theorem]
  \label{thm:bft}
  The optimal value function $\hat V_{*, \mdplen + 1}$ of a Markov decision process of length $\mdplen + 1$
  relates to the optimal value function of the same Markov decision process of length $\mdplen$
  by the inductive relation
  \begin{equation}
   \hat V_{*, \mdplen + 1} = \hat \bellman  V_{*, \mdplen}
 \end{equation}
 where $\hat \bellman$ is Bellman optimality operator (\rref{def:bellmanopt}). 
\end{theorem}

The iterative computation of the sequence of
optimal value functions $\{\hat V_{*, \mdplen}\}_{\mdplen \in \integers_{\geq 0}}$ of Markov decision processes  of length $\mdplen = 0, 1, 2, \dots$
from the recursion $\hat V_{*, \mdplen+1} = \hat \bellman \hat V_{*, \mdplen}$ is
the same algorithm as \emph{value  iteration}.

\subsection{Comments on Formalization}
\libname contributes a formal library for reasoning about Markov decision processes. We demonstrate the effectiveness of this library's building blocks by proving the two most canonical results from  reinforcement learning theory.
In this section we reflect on the structure of \libname's formalization, substantiating our claim that \libname serves as a convenient foundations for a continuing line of work on formalization of reinforcement learning theory.

\subsubsection{Characterizing Optimality}

Most texts on Markov decision processses (for example \cite[Section 2.1.6]{Puterman1994}) start out with a probability space on the space of all possible realizations of the Markov decision process. The long-term value for an infinite-horizon Markov decision process is then defined as the expected value over all possible realizations:
\begin{equation}
  V_\policy(s) = \mathbb{E}_{(x_1,x_2,\dots)}\left[ \sum_{\step=0}^{\infty}\gamma^\step v(x_\step,\policy(x_\step))| x_0 = s ; \policy\right]
\end{equation}
where each $x_\step$ is drawn from the distribution $T(x_{\step-1},\policy(x_{\step-1}))$. 
This definition is hard to work with because, as \cite{Puterman1994} notes, it ignores the dynamics of the problem. 
Fortunately, it is also unnecessary since statements about the totality of all realizations are rarely made. 

In our setup, following \cite{feys2018long}, we only consider the probability space over the finite set of states of the Markov decision process. 
By identifying the basic operation of Kleisli composition, we generate more realizations (and their expected rewards) on the fly as and when needed.

Implementations of reinforcement learning algorithms often compute the long-term value using matrix operators for efficiency reasons.
The observation that clean theoretical tools do not necessarily entail efficient implementations is not a new observation; both Puterman \cite{Puterman1994} and H\"{o}lzl \cite{hoelzl2017mdp} make similar remarks.
Fortunately, the design of our library provides a clean interface for future work on formalizing efficiency improvements.
Extending \libname with correctness theorems for algorithms that use matrix operations requires nothing more than a proof that the relevant matrix operations satisfy the definition of Kleisli composition.

\subsubsection{Comparison of English and Coq Proofs}

Comparing \rref{thm:ltv_Rfct_le_fixpt} and \rref{thm:exists_fixpt_policy} with the the equivalent results from Puterman \cite[Theorem 6.2.2]{Puterman1994} demonstrates that \libname avoids reasoning about low-level $\epsilon - \delta$ details through strategic use of coinduction. 

\lstnewenvironment{coqn}{\lstset{language=Coq, numbers=left, numberstyle=\tiny, numbersep=5pt}}{}
\begin{table*}[h!]  \centering
  \begin{subtable}[b]{.5\textwidth}
    \begin{center}
    \exists*
    \begin{proof} \ \\
    \begin{enumerate}
      \item $V_{\sigma_*} \le \hat V$ follows by \rref{thm:ltv_Rfct_le_fixpt}. 
      \item Now we have to show $\hat V \le V_{\sigma_*}$. Note that we have $V_{\sigma_*}$ is the fixed point of $B_{\sigma_*}$ by \rref{thm:bellman_op_prop}.
      \item We can now apply contraction coinduction with $F = \bellman_{\sigma_*}$.
      \item The hypotheses are satisfied since by \rref{thm:bellman_op_prop}, the $\bellman_{\sigma_*}$ is a contraction and it is a monotone operator.
      \item The only hypothesis left to show is $\hat V \le \bellman_{\sigma_*} \hat V$.
      \item But in fact, we have $\bellman_{\sigma_*}(\hat V) = \hat V$ by the definition of $\sigma_*$.
    \end{enumerate}
   \end{proof}
  \end{center}
  \caption{English proof adapted from \cite{feys2018long}.}
\end{subtable}
\hspace{.08\textwidth}
  \begin{subtable}[b]{.35\textwidth}
    \begin{center} 
      \begin{coqn}
Lemma exists_fixpt_policy  : forall init,
  let V'  := fixpt (bellman_max_op) in
  let pi' := greedy init in
  ltv gamma pi' = V' init.
Proof.
intros init V' pi';
eapply Rfct_le_antisym; split.
  - eapply ltv_Rfct_le_fixpt.
  - rewrite (ltv_bellman_op_fixpt _ init).
    apply contraction_coinduction_Rfct_ge'.
    + apply is_contraction_bellman_op.
    + apply bellman_op_monotone_ge.
    + unfold V', pi'.
       now rewrite greedy_argmax_is_max. 
Qed.
    \end{coqn}
  \end{center}
  \caption{Coq proof~\coqdef{converge.mdp}{exists_fixpt_policy}}
\end{subtable}
\caption{Comparison of English and Coq proofs of \rref{thm:exists_fixpt_policy}.}
\label{tab:comparison}
\end{table*}

The usefulness of contraction coinduction is reflected in the formalization, sometimes resulting in Coq proofs whose length is almost the same as the English text.

We compare in \rref{tab:comparison} the Coq proof of \rref{thm:exists_fixpt_policy} to an English proof of the same. The two proofs are roughly equivalent in length and, crucially, also make essentially the same argument at the same level of abstraction.
Note that what we compare is not \emph{exactly} the proof from Feys et al. \cite[Proposition 1]{feys2018long}, but is as close as possible to a restatement of their Proposition 1 and Lemma 1 with the proof of Lemma 1 inlined and the construction restated in terms of our development. The full proof from \cite{feys2018long}, with Lemma 1 inlined, reads as follows:
\begin{quote}
  \textbf{Proposition 1}: The greedy policy is optimal. That is, $LTV_{\sigma^*} = V^*$.
\begin{enumerate}
  \item   Observe that $\Psi_{\sigma^*} \ge V^*$ (in fact, equality holds).
  \item   By contraction coinduction, $V^* \le LTV_{\sigma^*}$.
  \item   \textbf{Lemma 1:} For all policies $\sigma$, $LTV_\sigma \le V^*$.
  \item   A straightforward calculation and monotonicity argument shows that for all $f \in B(S, \mathbb{R})$,
   $\Psi_\sigma(f) \le \Psi^*(f)$.
  \item   In particular, $LTV_\sigma = \Psi_\sigma(LTV_\sigma) \le \Psi^*(LTV_\sigma)$.
  \item   By contraction coinduction we conclude that $LTV_\sigma \le V^*$.
\end{enumerate}
\end{quote}

\rref{tab:comparison} compares two coinductive proofs -- one in English and the other in Coq. 
Another important comparison is between a Coq coinductive proof and an English non-coinductive proof.
The Coq proof of the policy improvement theorem provides one such point of comparison. Recall that theorem states that a particular closed property (the set $\{x | x \le y\}$) holds of the fixed point of a particular contractive map (the Bellman operator).  
The most common argument -- presented in the most common textbook on reinforcement learning -- proves this theorem by expanding the infinite sum in multiple steps \cite[Section 4.2]{sutton.barto:reinforcement}. We reproduce this below:
\begin{theorem}[Policy Improvement Theorem \cite{sutton.barto:reinforcement}]
  Let $\pi,\pi'$ be a pair of deterministic policies such that, for all states $s$, 
  \begin{equation}\label{eq:policyimp}
  Q_\pi(s,\pi'(s)) \ge V_{\pi}(s)
  \end{equation} 
  then $V_{\pi'}(s) \ge V_{\pi}(s)$. 
\end{theorem}
\begin{proof}
  Starting with (\ref{eq:policyimp}) we keep expanding the $Q_{\pi}$ side and reapplying (\ref{eq:policyimp}) until we get $V_{\pi'}(s)$. 
  \begin{align*}
    V_{\pi}(s) &\le Q_{\pi}(s,\pi'(s)) \\
    &= \E_{\pi'}\left\{ r_{t+1} + \gamma V_{\pi}(s_{t+1})|s_t = s\right\} \\
    &\le \E_{\pi'}\left\{ r_{t+1} + \gamma Q_{\pi}(s_{t+1},\pi'(s_{t+1}))| s_t = s\right\} \\
    &= \E_{\pi'}\left\{ r_{t+1} + \gamma \E_{\pi'}\left\{ r_{t+2} + \gamma V_{\pi}(s_{t+2}) \right\} | s_t = s\right\}\\
    &=\E_{\pi'}\left\{ r_{t+1} + \gamma r_{t+2} + \gamma^2 V_{\pi}(s_{t+2}) \ |\ s_t = s \right\}\\
    &\le \E_{\pi'}\left\{ r_{t+1} + \gamma r_{t+2} + \gamma^2 r_{t+3} + \gamma^3 V_{\pi}(s_{t+3}) \ |\ s_t = s \right\}\\
    &\vdots \\
    &\le \E_{\pi'}\left\{ r_{t+1} + \gamma r_{t+2} + \gamma^2 r_{t+3} + \gamma^3 r_{t+4} \dots \ |\ s_t = s \right\}\\
    &= V_{\pi'}(s)
  \end{align*}
\end{proof}

At a high level, this proof proceeds by showing that the closed property $\{x | x \le y\}$ holds of each partial sum of the infinite series $V_{\pi}(s)$. By completeness and using the fact that partial sums converge to the full series $V_{\pi}$, this property is also shown to hold of the fixed point $V_{\pi}$. 

However, the coinductive version of this proof (Theorem \ref{thm:policy_improvement}) is simpler because it exploits the fact that this construction has already been done once in the proof of the fixed point theorem: the iterates of the contraction operator were already proven to converge to the fixed point and so there is no reason to repeat the construction again. Thus, the proof is reduced to simply establishing the ``base case'' of the (co)induction. 

The power of this method goes beyond simplifying proofs for Markov decision processes. See \cite{KozenRuozzi} for other examples.

\section{Related and Future Work}
To our knowledge, \libname is the first formal proof of convergence for value iteration or policy iteration.
Related work falls into three categories: 
\begin{enumerate}
  \item libraries that \libname builds upon,  
  \item formalizations of results from probability and machine learning, and
  \item work at the intersection of formal verification and reinforcement learning.
\end{enumerate}

\paragraph{Dependencies}
\libname builds on the Coquelicot \cite{Coquelicot} library for real analysis. 
Our main results are statements about fixed points of contractive maps in complete normed modules.
\libname therefore builds on the formal development of the Lax-Milgram theorem and, in particular, Boldo et al.'s formal proof of the Banach fixed point theorem \cite{BoldoElfic}.
\libname also makes extensive use of some utilities from the Q*cert project \cite{DBLP:conf/sigmod/AuerbachHMSS17a}.
\libname includes a bespoke implementation of some basic results and constructions from probability theory and also an implementation of the Giry monad. Our use of the monad for reasoning about probabilistic processes, as well as the design of our library, is highly motivated by the design of the Polaris library \cite{tassarotti2019separation}. Many of the thorough formalizations of probabilities in Coq -- such as the Polaris \cite{tassarotti2019separation}, Infotheo \cite{affeldt2012itp}, and Alea \cite{audebaud2009proofs} -- also contain these results. Refactoring \libname to build on top of one or more of these formalizations might allow future work on certified reinforcement learning to leverage future improvements to these libraries.

Building on these other foundations, \libname demonstrates how existing work on formalization enables formalization of key results in reinforcement learning theory.

\paragraph{Related Formalizations}
There is a growing body of work on formalization of machine learning theory \cite{DBLP:journals/corr/abs-1911-00385,DBLP:journals/corr/abs-2007-06776,markovInHOL,DBLP:conf/icml/SelsamLD17,DBLP:conf/aaai/Bagnall019,DBLP:journals/jar/BentkampBK19}.

Johannes H\"{o}lzl's Isabelle/HOL development of Markov processes is most related to our own work \cite{markovInHOL,hoelzl2017mdp}.
H\"{o}lzl builds on the probability theory libraries of Isabelle/HOL to develop continuous-time Markov chains.
Many of H\"{o}lzl's basic design choices are similar to ours; for example, he also uses the Giry monad to place a monadic structure on probability spaces and also uses coinductive methods.
\libname focuses instead on formalization of convergence proofs for dynamic programming algorithms that solve Markov decision processes. In the future, we plan to extend our formalization to include convergence proofs for model-free methods, in which a fixed Markov decision process is not known \emph{a priori}.

The CertiGrad formalization by Selsam et al. contains a Lean proof that the gradients sampled by a stochastic computation graph are unbiased estimators of the true mathematical function \cite{DBLP:conf/icml/SelsamLD17}. 
This result, together with our development of a library for proving convergence of reinforcement learning algorithms, provides a path toward a formal proof of correctness for deep reinforcement learning.

\paragraph{Formal Methods for RL}
The likelihood that reinforcement learning algorithms will be deployed in safety-critical settings during the coming decades motivates a growing body of work on formal methods for safe reinforcement learning.
This approach -- variously called 
formally constrained reinforcement learning \cite{HasanbeigKroening}, 
shielding\cite{DBLP:conf/aaai/AlshiekhBEKNT18}, 
or
verifiably safe reinforcement learning \cite{Hunt2020VerifiablySE} 
-- uses temporal or dynamic logics to specify constraints on the behavior of RL algorithms.

Global convergence is a fundamental theoretical property of classical reinforcement learning algorithms, and in practice at least local convergence is an important property for any useful reinforcement learning algorithm.
However, the formal proofs underlying these methods typically establish the correctness of a safety constraint but do not formalize any convergence properties.
In future work, we plan to establish an end-to-end proof that constrained reinforcement learning safely converges by combining our current development with the safe RL approach of Fulton et al. \cite{aaai18} and the VeriPhy pipeline of Bohrer et al. \cite{DBLP:conf/pldi/BohrerTMMP18}.

\section{Conclusions}
Reinforcement learning algorithms are an important class of machine learning algorithms that are now being deployed in safety-critical settings.
Ensuring the correctness of these algorithms is societally important, but proving properties about stochastic processes presents several challenges.
In this paper we show how a combination of metric coinduction and the Giry monad provides a convenient setting for formalizing convergence proofs for reinforcement learning algorithms.

\begin{acks}
  We thank Larry Moss, Sylvie Boldo and Mark Squillante for discussions related to this paper. 
  Some of the work described in this paper was performed while Koundinya Vajjha was an intern at IBM Research. Vajjha was additionally supported by the Alfred P. Sloan Foundation under grant number G-2018-10067.
\end{acks}


\bibliographystyle{alpha}
\bibliography{main} 

\end{document}